\documentclass[lettersize,journal]{IEEEtran}
\usepackage{amsmath,amsfonts}
\usepackage{algorithmic}
\usepackage{algorithm}
\usepackage{array}
\usepackage[caption=false,font=normalsize,labelfont=sf,textfont=sf]{subfig}
\usepackage{textcomp}
\usepackage{stfloats}
\usepackage{url}
\usepackage{verbatim}
\usepackage{graphicx}
\usepackage{cite}
\usepackage{amsmath} 
\usepackage{amsthm}
\newtheorem{theorem}{Theorem}

\newtheorem{remark}{Remark}
\newtheorem{definition}{Definition}

\usepackage{multirow} 
\usepackage{algorithm}
\usepackage{algorithmic}
\usepackage{booktabs}
\usepackage{newtxtext}
\usepackage{placeins} 
\usepackage{hyperref}
\begin{document}

\title{Multi-Scale Harmonic Encoding for Feature-Wise Graph Message Passing}

\author{
    Longlong Li, Mengyang Zhao, Guanghui Wang and Cunquan Qu\thanks{*Corresponding author.}
    
    \thanks{Longlong Li is with the School of Mathematics and the Data Science Institute, Shandong University, Jinan 250100, China (e-mail: longlee@mail.sdu.edu.cn).}

    \thanks{Mengyang Zhao is with the School of Mathematics and the Data Science Institute, Shandong University, Jinan 250100, China (email: mengyangzhao@mail.sdu.edu.cn)}
    
    \thanks{Guanghui Wang is with the School of Mathematics, Shandong University, Jinan 250100, China (e-mail: ghwang@sdu.edu.cn).}
    
    \thanks{*Cunquan Qu is the corresponding author and is with the Data Science Institute, Shandong University, Jinan 250100, China (e-mail: cqqu@sdu.edu.cn).}
    
}

\maketitle

\begin{abstract}
Most Graph Neural Networks (GNNs) propagate messages by treating node embeddings as holistic feature vectors, implicitly assuming uniform relevance across feature dimensions. This limits their ability to selectively transmit informative components, especially when graph structures exhibit distinct frequency characteristics. We propose MSH-GNN (Multi-Scale Harmonic Graph Neural Network), a frequency-aware message passing framework that performs feature-wise adaptive propagation. Each node projects incoming messages onto node-conditioned feature subspaces derived from its own representation, enabling selective extraction of frequency-relevant components. Learnable multi-scale harmonic modulations further allow the model to capture both smooth and oscillatory structural patterns. A frequency-aware attention pooling mechanism is introduced for graph-level readout. We show that MSH-GNN admits an interpretation as a learnable Fourier-feature approximation of kernelized message functions and matches the expressive power of the 1-Weisfeiler--Lehman (1-WL) test. Extensive experiments on node- and graph-level benchmarks demonstrate consistent improvements over state-of-the-art methods, particularly in joint structure--frequency analysis tasks.
\end{abstract}

\begin{IEEEkeywords}
Graph Neural Networks, Message Passing, Frequency-Aware Learning,
Harmonic Representation, Feature-Wise Propagation
\end{IEEEkeywords}

\section{Introduction}

Graph Neural Networks (GNNs)~\cite{kipf2016semi, hamilton2017inductive, xu2018powerful,Zhang2024} have become a dominant paradigm for learning on graph-structured data, achieving strong performance in a wide range of domains, including molecular property prediction~\cite{gilmer2017neural}, protein modeling~\cite{gligorijevic2021structure}, and social network analysis~\cite{guo2020deep}. Most modern GNN architectures are built upon message passing, where each node iteratively updates its representation by aggregating information from its local neighborhood~\cite{wu2020comprehensive, Qi2025}. Despite their empirical success, existing message passing schemes predominantly focus on where information is aggregated from, while paying comparatively less attention to which components of the incoming features should be propagated. In many real-world graphs, node embeddings encode heterogeneous signals, such as structural, chemical, or functional attributes, whose relevance varies across nodes and local contexts. This mismatch between heterogeneous feature semantics and uniform propagation rules raises a fundamental question: can GNNs selectively propagate feature components in a context-dependent manner, rather than treating node embeddings as indivisible wholes?

Most conventional GNNs apply shared linear transformations to entire neighbor embeddings before aggregation, implicitly assuming uniform importance across all feature dimensions. This holistic message passing paradigm does not explicitly differentiate which feature components or subspaces are most informative for a specific target node. For instance, in molecular graphs, different feature dimensions may correspond to geometric configurations, chemical environments, or interaction patterns, whose relevance strongly depends on local structural context. Attention-based models, such as GAT~\cite{velivckovic2017graph}, GATv2~\cite{brody2021attentive}, GAAN~\cite{zhang2018gaan}, and their extensions~\cite{rong2019dropedge, ying2021transformers, wu2021representing}, introduce adaptivity by assigning scalar attention weights to neighbors. However, these methods primarily modulate neighbor contributions at the node level; once a neighbor is selected, its full embedding is typically aggregated without further feature-wise differentiation. As a result, they are not explicitly designed to capture variations in the internal composition of neighbor features, which can be important for modeling subspace-sensitive or frequency-dependent patterns.

Recent efforts have explored alternatives to holistic aggregation by incorporating channel-wise modulation or spectral representations, such as GNN-FiLM~\cite{brockschmidt2020gnn}, GraphGPS~\cite{rampavsek2022recipe}, and SpecFormer~\cite{bo2023specformer}. These approaches enhance model flexibility through global conditioning or fixed spectral encodings. However, they primarily focus on global or layer-wise modulation, and do not explicitly address localized, node-dependent feature projection during message construction. In particular, how a target node should reinterpret incoming features based on its own representation and local structural context remains an open modeling question. This limitation is illustrated in Figure~\ref{fig_1:gat-failure}. Although graphs $G_1$ and $G_2$ share similar average neighbor features, they differ in how these features are composed across feature dimensions. GAT assigns scalar attentions to entire neighbor embeddings and aggregates them uniformly, making it challenging to distinguish such differences in the feature-wise composition of neighbor representations. In contrast, a feature-wise message passing mechanism can capture these discrepancies by selectively projecting neighbor features into node-specific feature subspaces.
\begin{figure}[htbp]
    \centering
    \includegraphics[width=0.6\linewidth]{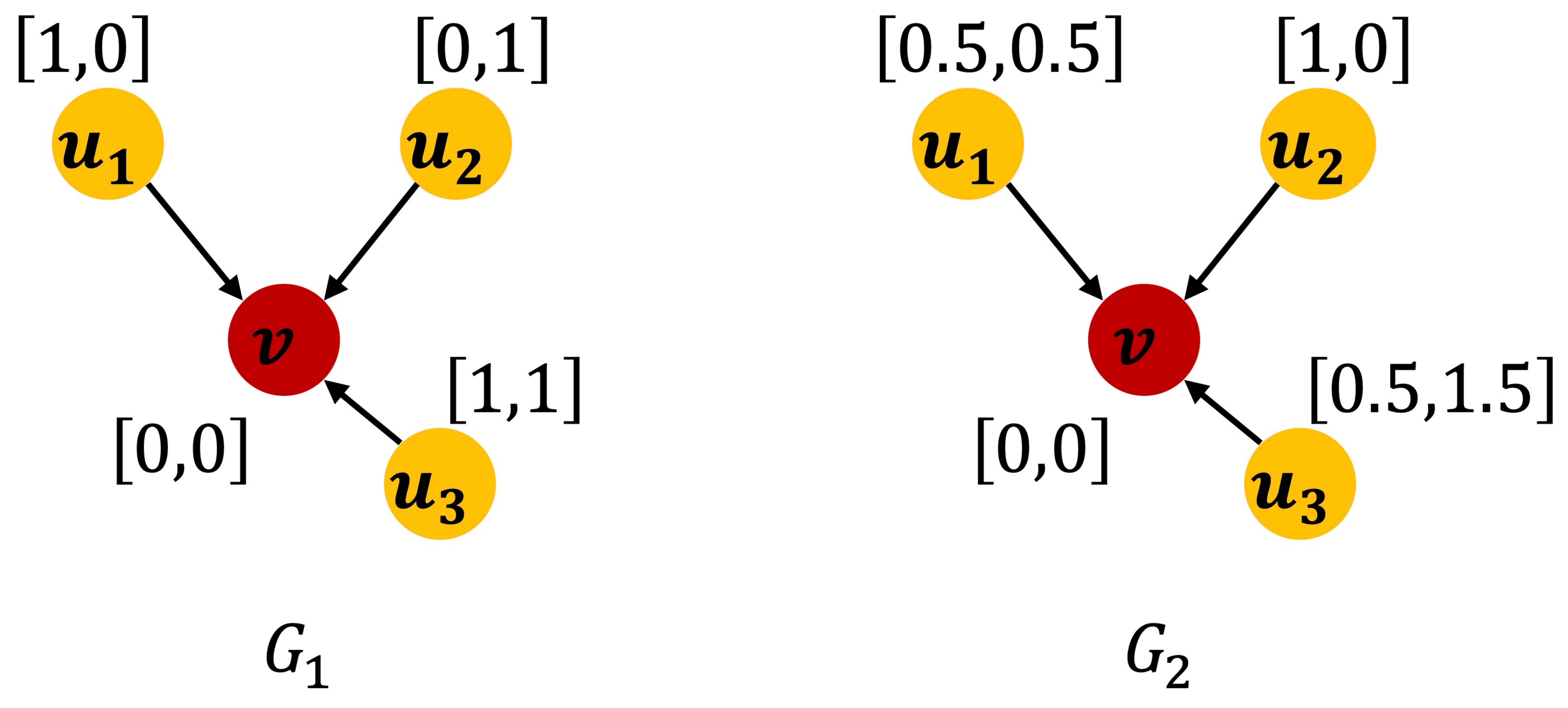}
    \caption{
Comparison of neighbor feature aggregation behaviors.
Although graphs $G_1$ and $G_2$ have similar average neighbor features, they differ in how these features are composed across feature dimensions.
Attention-based aggregation assigns scalar weights to entire neighbor embeddings, which can obscure such differences in feature composition.
By contrast, MSH-GNN applies node-conditioned feature-wise projections, enabling these compositional variations to be preserved during aggregation.
See Appendix~\ref{appendix:gat-gap} for further analysis.
}
    \label{fig_1:gat-failure}
\end{figure}

To address these challenges, we propose MSH-GNN (Multi-Scale Harmonic Graph Neural Network), a novel architecture (Figure~\ref{fig_2:MSH-GNN}) that performs feature-wise adaptive message passing. Instead of aggregating neighbor embeddings as indivisible vectors, MSH-GNN allows each target node to project incoming features onto node-conditioned harmonic projection components derived from its own representation. These projections are further modulated through multi-scale sinusoidal functions, enabling the model to capture both smooth and oscillatory structural patterns across different frequency bands. A frequency-aware attention pooling mechanism aggregates node representations into graph-level embeddings by emphasizing harmonically informative responses. Together, these components form a unified framework that bridges spatial message passing and harmonic (frequency-aware) inductive biases.
\begin{figure*}[htbp]
    \centering
    \includegraphics[width=0.9\linewidth]{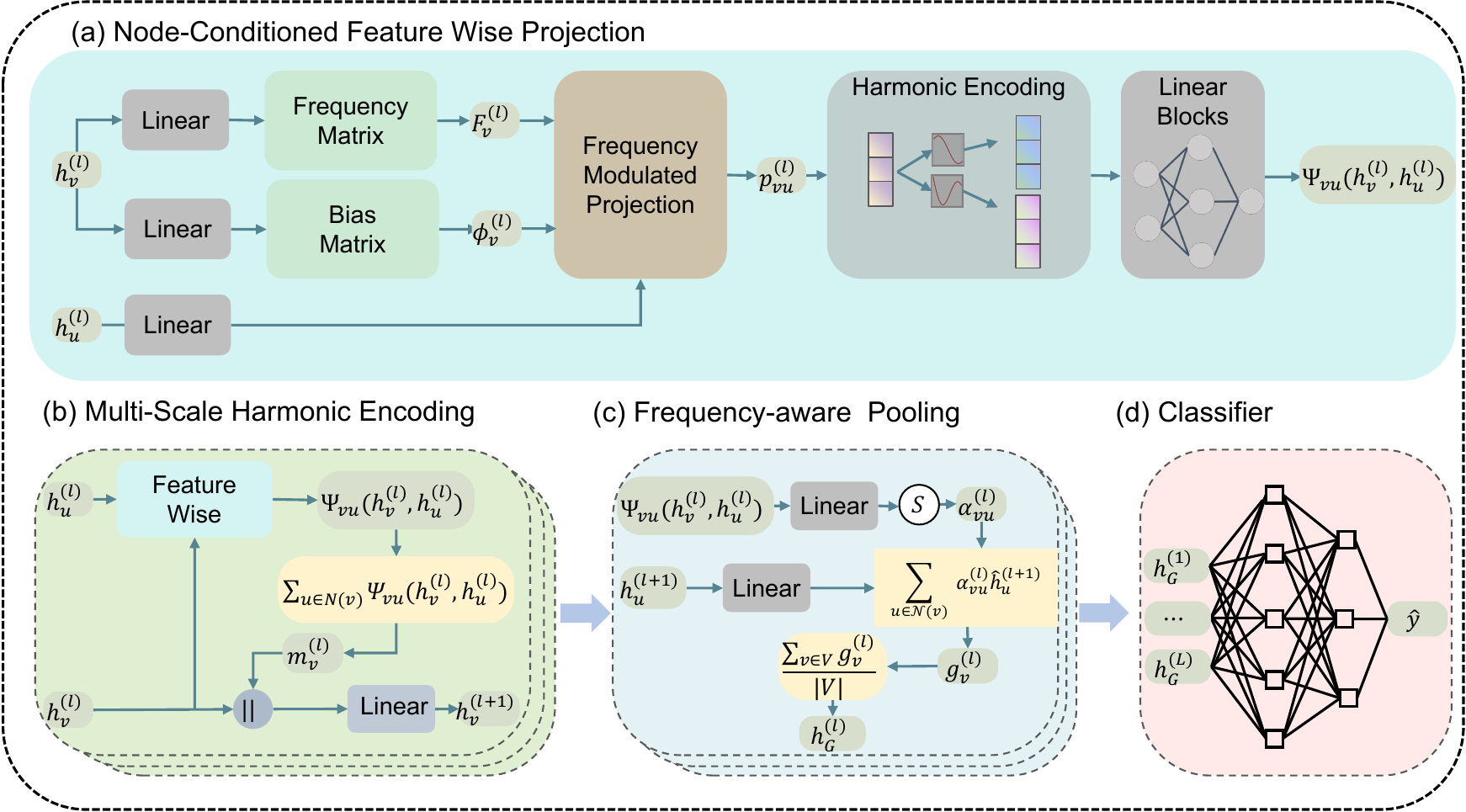}
    \caption{
Overview of the MSH-GNN architecture.
Incoming neighbor features are first projected into node-conditioned projection components and modulated across multiple harmonic scales to guide feature-wise message passing.
The resulting node representations are subsequently pooled in a frequency-aware manner and used for downstream prediction.
}
    \label{fig_2:MSH-GNN}
\end{figure*}
Our contributions are summarized as follows:
\begin{itemize}
    \item We identify a fundamental limitation of holistic message passing in existing GNNs and propose a feature-wise adaptive propagation paradigm that enables context-dependent feature selection.
    \item We develop a frequency-aware message passing framework based on node-specific harmonic projections with multi-scale sinusoidal modulation, enabling feature-sensitive and scale-adaptive aggregation.
    \item Extensive experiments on molecular, protein, and social network benchmarks demonstrate that MSH-GNN consistently outperforms state-of-the-art methods, particularly in joint structure--frequency discrimination tasks.
\end{itemize}

\section{Related Work}

\subsection{Holistic and Feature-Wise Message Passing in GNNs}

Message passing is the fundamental operation underlying most GNN architectures.
Early models such as GCN~\cite{kipf2016semi}, GraphSAGE~\cite{hamilton2017inductive}, and GIN~\cite{xu2018powerful} aggregate and transform neighbor embeddings in a holistic manner, treating node representations as indivisible feature vectors.
While effective in many settings, this paradigm applies uniform transformations across feature dimensions and does not explicitly model which components of neighbor features are most relevant to a given target node.

To improve adaptivity, several works have explored feature-wise or channel-wise modulation mechanisms. GNN-FiLM~\cite{brockschmidt2020gnn} introduces affine transformations conditioned on global or node-level information, enabling limited feature modulation but remaining largely detached from local structural context. Other models incorporate fixed or global spectral encodings, such as SpecFormer~\cite{bo2023specformer}, which uses Laplacian eigenvectors to guide downstream prediction. However, these approaches primarily operate at the layer or readout level, and do not explicitly perform node-conditioned feature selection during message construction.

Geometric, directional, and anisotropic GNNs~\cite{beaini2021directional,bodnar2021weisfeiler, wang2021weisfeiler} encode spatial or geometric orientations (e.g., edge directions in coordinate space) through predefined equivariant bases. While effective in geometric settings, these representations are typically tied to explicit geometry and are often fixed rather than target-conditioned in the learned
feature space.

MSH-GNN performs node-conditioned projections and harmonic modulation purely in the learned feature space, enabling localized feature-wise selection without assuming any geometric coordinate system.

\subsection{Attention-Based and Tokenized Graph Models}

Attention mechanisms have been widely adopted to enhance the expressiveness of message passing. Models such as GAT~\cite{velivckovic2017graph}, GATv2~\cite{brody2021attentive}, Graphormer~\cite{ying2021transformers} and GRARF~\cite{Ma2023} assign adaptive weights to neighbors, allowing nodes to emphasize more informative connections. These methods primarily focus on neighbor-level importance, while still aggregating entire feature vectors without explicit feature-wise differentiation.

Transformer-based graph models~\cite{shi2020masked, chen2022nagphormer, shirzad2023exphormer}, including TokenGT~\cite{kim2022pure} and GraphGPS~\cite{rampavsek2022recipe}, treat nodes or structural anchors as tokens and apply global attention mechanisms. These approaches are effective in capturing long-range dependencies, but typically rely on token-level interactions and do not explicitly shape local message content in a feature-wise or frequency-aware manner during propagation.

MSH-GNN is complementary to these models in that it focuses on a different axis of adaptivity: rather than selecting which neighbors or tokens to attend to, it selectively filters \emph{which feature components} are propagated through node-conditioned, frequency-aware projections.

\subsection{Spectral Models, Kernel Methods, and Graph Pooling}

Spectral GNNs design filters in the frequency domain using graph Laplacian eigenbases~\cite{jing2018waveletnet, ruiz2022graph}.
These methods capture global structural patterns through fixed or learned spectral representations, but typically operate on static bases and lack node-dependent adaptivity during message passing.

\noindent\textbf{Terminology.}
In this paper, ``frequency-aware'' refers to multi-scale harmonic modulation of
\emph{node-conditioned feature projections} (Fourier-feature style), rather than
spectral filtering explicitly parameterized on the Laplacian eigenbasis.

Kernel-based approaches, including graph kernels~\cite{kriege2020survey} and neural tangent kernels~\cite{du2019graph}, offer theoretical insights into graph expressiveness via fixed similarity measures or random feature approximations~\cite{rahimi2007random}.
While expressive, such methods are generally non-parametric or detached from end-to-end message passing architectures.

Graph pooling and readout mechanisms aim to aggregate node representations into graph-level embeddings.
Standard pooling methods, such as global mean pooling or attention-based pooling, summarize node features primarily based on permutation-invariant aggregation or learned importance scores, without explicitly accounting for spectral or frequency-specific responses~\cite{baek2021accurate}.
More advanced pooling strategies introduce hierarchical compression~\cite{wu2022structural}, topological criteria~\cite{chen2023topological}, or geometric partitioning~\cite{liu2025graph, zhang2025rhomboid} to capture higher-level structural abstractions.
However, these approaches generally focus on structural summarization or node selection, and remain largely decoupled from the feature-frequency characteristics induced during message passing.

MSH-GNN employs a frequency-aware pooling strategy that aligns with its feature-wise message passing mechanism, emphasizing nodes that exhibit informative harmonic responses across multiple scales.

\section{Methods}
\label{sec:methods}

\begin{definition}[Feature-Wise Projection Message Passing]
\label{def:fwmp}
A message passing scheme is said to perform \emph{feature-wise projection} if, for each target node $v$, incoming neighbor features $\mathbf{h}_u$ are first projected onto a set of node-conditioned feature subspaces generated from $\mathbf{h}_v$, and message construction operates on these projected components rather than on the original feature vectors as indivisible wholes. Formally, the message from node $u$ to node $v$ takes the form
\[
\Psi_{vu} = \Psi\big(\mathbf{F}_v \mathbf{h}_u + \boldsymbol{\phi}_v\big),
\]
where $\mathbf{F}_v$ and $\boldsymbol{\phi}_v$ are learnable functions of the target node representation $\mathbf{h}_v$.
\end{definition}

\begin{remark}
Throughout this paper, \emph{feature-wise} and \emph{projection-based} refer to operations in the learned feature space rather than geometric or spatial orientations.
\end{remark}

\subsection{MSH-GNN Architecture}
We propose MSH-GNN, a graph neural network that performs feature-wise adaptive message passing with localized frequency modulation. Unlike conventional GNNs that treat neighbor embeddings as holistic feature vectors, MSH-GNN enables each target node to extract task-relevant components from neighbors via node-conditioned feature projections. These projections are further modulated by multi-scale harmonic functions, allowing the model to capture both smooth and oscillatory structural patterns. Figure~\ref{fig_2:MSH-GNN} illustrates the overall architecture, which consists of: (1) node-conditioned feature-wise projection, (2) multi-scale harmonic modulation for frequency-aware message construction, and (3) frequency-aware pooling for graph-level readout.

\subsection{Node-conditioned Feature-wise Projection}
Let $G=(V,E)$ be a graph with node representations $\mathbf{h}_v^{(l)}\in\mathbb{R}^d$ at layer $l$. MSH-GNN equips each target node $v$ with a node-conditioned projection matrix and a \emph{phase shift} vector, both generated from $\mathbf{h}_v^{(l)}$:
\begin{equation}
\begin{split}
\mathbf{F}_v^{(l)} &= \mathrm{reshape}\!\left(\mathbf{W}_f \mathbf{h}_v^{(l)} + \mathbf{b}_f\right)\in \mathbb{R}^{F\times d},\\
\boldsymbol{\phi}_v^{(l)} &= \mathbf{W}_\phi \mathbf{h}_v^{(l)} + \mathbf{b}_\phi \in \mathbb{R}^{F}.
\end{split}
\end{equation}
Here, $\mathbf{W}_f\in\mathbb{R}^{Fd\times d}$ and $\mathbf{b}_f\in\mathbb{R}^{Fd}$ produce $F$ node-conditioned projection vectors, which are reshaped into $\mathbf{F}_v^{(l)}$; meanwhile $\mathbf{W}_\phi\in\mathbb{R}^{F\times d}$ and $\mathbf{b}_\phi\in\mathbb{R}^{F}$ generate the corresponding
phase shifts $\boldsymbol{\phi}_v^{(l)}$.

For each directed edge $(u\rightarrow v)$, we project the source features $\mathbf{h}_u^{(l)}$ into the node-conditioned feature subspace of $v$:
\begin{equation}
\mathbf{p}_{vu}^{(l)} = \mathbf{F}_v^{(l)} \mathbf{h}_u^{(l)} + \boldsymbol{\phi}_v^{(l)} \in \mathbb{R}^{F}.
\end{equation}
Unlike shared linear transformations, $\mathbf{F}_v^{(l)}$ and $\boldsymbol{\phi}_v^{(l)}$ are determined by the target node $v$, so the resulting projected representation $\mathbf{p}_{vu}^{(l)}$ is explicitly receiver-conditioned. This provides a feature-wise mechanism for $v$ to reinterpret and selectively emphasize components of incoming neighbor information before harmonic modulation and aggregation.

\subsection{Harmonic Encoding Module}

To nonlinearly enrich the node-conditioned feature projection $\mathbf{p}_{vu}^{(l)} \in \mathbb{R}^F$, we apply sinusoidal encoding over a set of predefined frequencies. Specifically, we define a set of $K$ harmonic frequencies $\{\omega_k\}_{k=1}^K$
and compute as:
\begin{equation}
\psi_k(\mathbf{p}_{vu}) =
\left[\sin(\omega_k \mathbf{p}_{vu}) \,\|\, \cos(\omega_k \mathbf{p}_{vu})\right],
\quad \omega_k \in \mathbb{R}.
\end{equation}

In this work, we adopt an exponential frequency schedule $\omega_k = 2^{k-1}$ with $k=1,2,3$, corresponding to $\omega=\{1,2,4\}$. This choice captures both low- and high-frequency components and empirically balances expressiveness and stability, similar in spirit to the hierarchical frequency scaling used in positional encodings for Transformers~\cite{vaswani2017attention}. Other frequency schedules are possible and left for future exploration.

This harmonic encoding scheme is reminiscent of random Fourier features~\cite{rahimi2007random}, which approximate shift-invariant kernels via trigonometric mappings. However, unlike randomly sampled bases, the encoding in MSH-GNN is applied to node-conditioned projections $\mathbf{p}_{vu}$, resulting in localized and
structure-aware feature representations.

The concatenation of sine and cosine functions across multiple frequencies provides a discriminative embedding that is designed to preserve variations across different feature-space projections. As a result, differences in $\mathbf{p}_{vu}$ across different projection components can be effectively reflected after harmonic modulation.

The final modulated edge message is computed as:
\begin{equation}
\Psi_{vu}(\mathbf{h}_v^{(l)}, \mathbf{h}_u^{(l)}) =
\mathbf{W}_o \left[
\psi_1(\mathbf{p}_{vu}^{(l)}) \,\|\, 
\psi_2(\mathbf{p}_{vu}^{(l)}) \,\|\, 
\psi_3(\mathbf{p}_{vu}^{(l)})
\right],
\end{equation}
where $\mathbf{W}_o$ is a learnable projection applied after the harmonic expansion.

\paragraph*{Kernel Interpretation}
The harmonic encoding can be interpreted as a Fourier-feature approximation of kernelized message functions. By applying multi-frequency sinusoidal mappings to node-conditioned feature projections, MSH-GNN induces target-dependent similarity functions during message passing, which resemble shift-invariant kernels defined in the projected feature space. A formal analysis is provided in Section~\ref{sec:theory}.

\subsection{Frequency-Aware Aggregation Layer}
Node $v$ aggregates frequency-modulated projection-based messages as:
\begin{equation}
    \mathbf{m}_v^{(l)} = \sum_{u \in \mathcal{N}(v)} \Psi_{vu}(\mathbf{h}_v^{(l)}, \mathbf{h}_u^{(l)}), 
\quad 
\mathbf{h}_v^{(l+1)} = \mathrm{MLP}(\mathbf{h}_v^{(l)} + \mathbf{m}_v^{(l)}).
\end{equation}

\subsection{Frequency-Aware Pooling}

To generate graph-level representations, we apply a frequency-aware attention mechanism to modulated edge messages. For each node $v$, we compute edge-level attention weights and intermediate pooled features:
\begin{equation}
\begin{split}
    \alpha_{vu}^{(l)} = \sigma(\mathbf{W}_1 \Psi_{vu}(\mathbf{h}_v^{(l)}, \mathbf{h}_u^{(l)})), \\
    \mathbf{g}_v^{(l)} = \sum_{u \in \mathcal{N}(v)} 
\alpha_{vu}^{(l)} \cdot \mathbf{W}_2 \Psi_{vu}(\mathbf{h}_v^{(l)}, \mathbf{h}_u^{(l)}).
\end{split}
\end{equation}
where $\mathbf{W}_1, \mathbf{W}_2$ are learnable parameters, and $\sigma$ denotes a non-linear activation function (e.g., $\mathrm{softmax}$ or $\mathrm{sigmoid}$) applied across edges.

The node-wise outputs are then averaged to form the layer-specific graph representation:
\begin{equation}
\mathbf{h}_G^{(l)} = \frac{1}{|V|} \sum_{v \in V} \mathbf{g}_v^{(l)}.
\end{equation}

After $L$ message passing layers, we perform a weighted combination of multi-layer representations to obtain the final graph embedding:
\begin{equation}
\mathbf{h}_G = \sum_{l=1}^L w_l \cdot \mathbf{h}_G^{(l)},
\end{equation}
where the weights $w_l$ are learnable scalars trained to adaptively fuse information across different depths.

\section{Theoretical Justification}
\label{sec:theory}

Unlike fixed Fourier positional encodings commonly adopted in Transformer-based graph models, the proposed harmonic mapping in MSH-GNN operates on \emph{node-conditioned feature projections}, thereby inducing an adaptive kernel structure that varies across target nodes and edges. This design goes beyond global or static encodings by enabling localized, context-dependent similarity modeling during message passing.

We analyze the expressive power of MSH-GNN from two complementary perspectives:  
\begin{enumerate}
    \item its interpretation as a kernelized message-passing mechanism via deterministic harmonic features; 
    \item its ability to distinguish non-isomorphic graphs under the 1-Weisfeiler--Lehman (1-WL) test.
\end{enumerate}

Formal derivations and proofs supporting these results are provided in Appendix~\ref{appendix:kernel-approximation}.

\paragraph{Kernel-based Expressiveness}

The harmonic encoding module in MSH-GNN combines node-adaptive linear projections with
multi-scale sinusoidal mappings, which can be interpreted as inducing
target-dependent similarity functions during message passing.
Specifically, each message takes the form
\begin{equation}
    \Psi_{vu}(\mathbf{h}_v, \mathbf{h}_u)
    = \mathbf{W}_o \cdot \psi\big(\mathbf{F}_v \mathbf{h}_u + \boldsymbol{\phi}_v\big),
\end{equation}
where $\mathbf{F}_v \in \mathbb{R}^{F \times d}$ and $\boldsymbol{\phi}_v \in \mathbb{R}^F$
are node-specific projections derived from $\mathbf{h}_v$, and $\psi$ denotes a
multi-frequency harmonic mapping.

From a functional perspective, the inner product between harmonic encodings
\begin{equation}
    \psi(\mathbf{p}_{vu})^\top \psi(\mathbf{p}_{vu'})
    \quad \text{with} \quad
    \mathbf{p}_{vu} = \mathbf{F}_v \mathbf{h}_u + \boldsymbol{\phi}_v,
\end{equation}
resembles structured Fourier feature mappings that approximate shift-invariant kernels \emph{on node-conditioned projected features}~\cite{rahimi2007random}. Under sufficiently rich frequency choices, such mappings are known to approximate a broad class of continuous functions over projected features~\cite{rahimi2007random, barron1993universal}.

We next analyze the combinatorial expressiveness of MSH-GNN from the perspective
of the 1-Weisfeiler--Lehman (1-WL) test, which characterizes the discriminative
power of a broad class of message passing GNNs.

\paragraph{Combinatorial Expressiveness via 1-WL}

To analyze the combinatorial expressiveness of MSH-GNN, we study its ability to distinguish non-isomorphic graphs under the 1-Weisfeiler–Lehman (1-WL) test. Unlike conventional GNNs that aggregate neighbor embeddings holistically, MSH-GNN performs feature-wise and frequency-modulated message passing, enabling more expressive multiset encodings under permutation-invariant aggregation.

\begin{theorem}[Sufficient Conditions for 1-WL-Level Expressiveness]
\label{the:1-wl}
Let $\mathcal{F}: G \rightarrow \mathbb{R}^n$ be an instance of MSH-GNN. Suppose:
\begin{itemize}
    \item Initial node features are distinguishable (e.g., injective or augmented with structural encodings);
    \item The projection $\mathbf{F}_v$ and phase shift $\boldsymbol{\phi}_v$ are generated via universal MLPs;
    \item The harmonic encoder $\psi$ uses distinct, non-zero frequencies $\{\omega_k\}$;
    \item The aggregation and readout functions are permutation-invariant and sufficiently expressive (e.g., MLP).
\end{itemize}
Then, $\mathcal{F}$ is at least as powerful as the 1-WL test in distinguishing non-isomorphic graphs.
\end{theorem}

\begin{remark}
Even when the assumptions above are only approximately satisfied (e.g., MLPs not perfectly injective), the harmonic projection and modulation still provide a strong inductive bias toward structure-aware neighborhood encoding. In practice, universal approximation and frequency separation ensure that MSH-GNN closely matches 1-WL behavior on real-world tasks.
\end{remark}

\subsection{Computational Complexity}

Let $F$ be the projection dimension, $d$ the feature dimension, and $K$ the number of harmonic frequencies. For each edge, the projection and harmonic modulation cost $\mathcal{O}(Fd + FK)$. Given $M$ edges and $N$ nodes, the per-layer cost is
\[
\mathcal{O}\big(M(Fd+FK) + N F d\big) = \mathcal{O}\big((M+N)Fd + MFK\big).
\]
Thus, the total cost over $L$ layers is
\[
\mathcal{O}\big(L((M+N)Fd + MFK)\big),
\]
where $K$ is typically small in practice.
All operations are parallelizable on GPUs, and for small graphs (e.g., molecular datasets), this cost is negligible relative to transformer models.

\section{Experiments}
We assess the performance of MSH-GNN across diverse graph domains and tasks to verify its generality and frequency-aware design advantages. Our experiments address both node- and graph-level predictions, covering both graphs with regular substructures (e.g., molecular and protein graphs) and large-scale real-world graphs (e.g., social networks and synthetic benchmarks).

Our evaluation covers two representative graph families:
\begin{itemize}
    \item {Molecular and protein graphs} (TU datasets; e.g., MUTAG, PTC, PROTEINS, NCI1), where recurrent biochemical motifs (rings, functional groups, folds) are prevalent;
    \item {Social and collaboration graphs} (e.g., IMDB-B, IMDB-M, RDT-B, COLLAB), where topology is diverse and node attributes are often absent or weak.
\end{itemize}

\subsection{Baselines}
We compare {MSH-GNN} with representative baselines from three lines of work:
(i) kernel-based graph classification, (ii) expressive message-passing GNNs, and
(iii) pooling/hierarchical readout models. Concretely, we include:
\begin{itemize}
    \item {Kernel-based methods}: GNTK~\cite{du2019graph} and DCNN~\cite{atwood2016diffusion}.
    \item {Message-passing GNNs}: DGCNN~\cite{zhang2018end}, GIN~\cite{xu2018powerful}, IGN~\cite{maron2018invariant}, 
    PPGNs~\cite{maron2019provably}, Natural GN~\cite{de2020natural}, GSN~\cite{bouritsas2022improving}, 
    SIN~\cite{bodnar2021weisfeiler}, CIN~\cite{wang2021weisfeiler}, PIN~\cite{truong2024weisfeiler}, and N$^2$~\cite{sun2024towards}.
    \item {Graph pooling/readout models}: GMT~\cite{baek2021accurate}, SEP~\cite{wu2022structural}, Wit-TopoPool~\cite{chen2023topological}, 
    GrePool~\cite{liu2025graph}, and RTPool~\cite{zhang2025rhomboid}.
\end{itemize}

\subsection{Graph Classification}
We follow the standard TU benchmark protocol using 10-fold cross-validation and report the mean accuracy with standard deviation. All baselines follow their original implementation settings. For datasets without node attributes (e.g., IMDB-B, IMDB-M, COLLAB), we adopt the common practice of encoding node degrees as one-hot vectors~\cite{xu2018powerful}.

\begin{table*}[t]
\centering
\resizebox{\textwidth}{!}{
\begin{tabular}{lcccccccccc}
\toprule
\multicolumn{2}{c}{\textbf{Methods}} & \textbf{MUTAG} & \textbf{PTC} & \textbf{PROTEINS} & \textbf{NCI1}  & \textbf{IMDB-B} & \textbf{IMDB-M} & \textbf{RDT-B} & \textbf{COLLAB} \\
\midrule
\multicolumn{2}{c}{\textit{\# Graphs}}  & 188 & 344 & 1113 & 4110 & 1000 & 1500 & 2000 & 5000 \\
\multicolumn{2}{c}{\textit{Avg. Nodes}} & 17.93 & 25.56 & 39.06  & 29.68 & 19.77 & 13.00 & 429.63 & 74.49 \\
\multicolumn{2}{c}{\textit{Avg. Edges}}& 19.79 & 26.96 & 72.82  & 32.13 & 96.53 & 65.94 & 497.75 & 2457.78 \\
\multicolumn{2}{c}{\textit{\# Classes}}  & 2 & 2 & 2 & 2 & 2  & 3 & 2 & 3 \\
\midrule
\multirow{2}{*}{Graph Kernels}
&GNTK~\cite{du2019graph} & 90.0$\pm$8.5 & 67.9$\pm$6.9 & 75.6$\pm$4.2 & 84.2$\pm$1.5 & 76.9$\pm$3.6 & 52.8$\pm$4.6 & N/A & 83.6$\pm$1.0 \\
&DCNN~\cite{atwood2016diffusion} & N/A & N/A & 61.3$\pm$1.6 & 56.6$\pm$1.0 & 49.1$\pm$1.4 & 33.5$\pm$1.4 & N/A & 52.1$\pm$0.7 \\
\midrule
\multirow{10}{*}{GNNs} 
&DGCNN~\cite{zhang2018end} & 85.8$\pm$1.8 & 58.6$\pm$2.2 & 75.5$\pm$0.9 & 74.4$\pm$0.5& 70.0$\pm$0.9 & 47.8$\pm$0.9 & N/A & 73.8$\pm$0.5 \\
&IGN~\cite{maron2018invariant} & 83.9$\pm$13.0 & 58.5$\pm$6.6 & 76.6$\pm$5.5 & 74.3$\pm$2.7 & 72.0$\pm$5.5 & 48.7$\pm$3.4 & N/A & 78.3$\pm$2.5 \\
&GIN~\cite{xu2018powerful} & 89.4$\pm$5.6 & 64.6$\pm$1.0 & 76.2$\pm$2.8 & 82.7$\pm$1.7 & 75.1$\pm$5.1 & 52.3$\pm$2.8 & 92.4$\pm$2.5 & 80.2$\pm$1.9 \\
&PPGNs~\cite{maron2019provably} & 90.6$\pm$8.7 & 66.2$\pm$6.6 & 77.2$\pm$4.7 & 83.2$\pm$1.1  & 73.0$\pm$5.8 & 50.5$\pm$3.6 & N/A & 81.4$\pm$1.4 \\
&Natural GN~\cite{de2020natural} & 89.4$\pm$1.6 & 66.8$\pm$1.1 & 71.7$\pm$1.0 & 82.4$\pm$1.3& 73.5$\pm$2.0 & 51.3$\pm$1.5 & N/A & N/A \\
&GSN~\cite{bouritsas2022improving} & 92.2$\pm$7.5 & 68.2$\pm$2.5 & 76.6$\pm$5.0 & 83.5$\pm$2.0  & 77.8$\pm$3.3 & 54.3$\pm$3.3 & N/A & 85.5$\pm$1.2 \\
&SIN~\cite{bodnar2021weisfeiler} & N/A & N/A & 76.4$\pm$3.3 & 82.7$\pm$2.1 & 75.6$\pm$3.2 & 52.4$\pm$2.9 & 92.2$\pm$1.0 & N/A \\
&CIN~\cite{wang2021weisfeiler} & {92.7$\pm$6.1} & {68.2$\pm$5.6} & 77.0$\pm$4.3 & 83.6$\pm$1.4 & 75.6$\pm$3.7 & 52.7$\pm$3.1 & 92.4$\pm$2.1 & N/A \\
&PIN~\cite{truong2024weisfeiler}&N/A&N/A&78.8$\pm$4.4&85.1$\pm$1.5&76.6$\pm$2.9&N/A&N/A&N/A\\
&N$^2$~\cite{sun2024towards} & N/A & N/A & 77.5$\pm$1.8 & 83.5$\pm$3.8  & 80.0$\pm$2.5 & 57.3$\pm$2.2 & N/A & 86.7$\pm$1.6 \\
\midrule
\multirow{3}{*}{Pooling} 
&GMT~\cite{baek2021accurate} & 83.4$\pm$1.3 & N/A & 75.1$\pm$0.6 & N/A  & 73.5$\pm$0.8 & 50.7$\pm$0.8 & N/A & 80.7$\pm$0.5 \\
&SEP~\cite{wu2022structural} & 85.6$\pm$1.1 & N/A & 76.4$\pm$0.4 & 78.4$\pm$0.3 & 74.1$\pm$0.6 & 51.5$\pm$0.7 & N/A & 81.3$\pm$0.2 \\
&Wit-TopoPool~\cite{chen2023topological} & 93.2$\pm$4.1 & 70.6$\pm$4.4 & 80.0$\pm$3.2  &N/A& 78.4$\pm$1.5 & 53.3$\pm$2.4 & 92.8$\pm$1.10 & N/A \\
&GrePool~\cite{liu2025graph} & 86.3$\pm$8.4 & 59.9$\pm$6.7 & N/A & 82.1$\pm$1.6 & N/A & 50.8$\pm$3.5 & N/A & 81.4$\pm$1.5 \\
&RTPool~\cite{zhang2025rhomboid}&94.7$\pm$3.3&76.6$\pm$1.1&N/A&N/A&73.0$\pm$3.8& 53.3$\pm$1.2&N/A&N/A&\\
\midrule
\multicolumn{2}{c}{\textbf{MSH-GNN}}& \textbf{99.1$\pm$0.3} &\textbf{91.4$\pm$1.5} &\textbf{94.1$\pm$3.3}&\textbf{88.6$\pm$0.5}   &\textbf{88.6$\pm$3.3}&\textbf{61.6$\pm$0.4}& \textbf{95.8$\pm$0.9} &\textbf{96.4$\pm$0.7}\\
\bottomrule
\end{tabular}}
\caption{Graph classification results on TU datasets.
Results are reported as mean accuracy (\%) $\pm$ standard deviation.
Methods are grouped into kernel-based models, message-passing GNNs,
and pooling/readout architectures.}
\label{tab:tu-results}
\end{table*}

As reported in Table~\ref{tab:tu-results}, MSH-GNN achieves consistently strong performance across all TU benchmarks and outperforms all competing baselines on most datasets.

On motif-rich molecular datasets such as MUTAG, PTC, and PROTEINS, MSH-GNN exhibits particularly pronounced gains, reaching $99.10\%$ accuracy on MUTAG and improving upon strong GNN baselines (e.g., GIN, GSN, CIN). These datasets are characterized by recurrent biochemical motifs (e.g., aromatic rings, chains, and functional groups) that induce regular and frequency-sensitive structural patterns. The observed improvements are consistent with the design of MSH-GNN, where node-conditioned harmonic projections enable selective extraction of feature components associated with such recurring substructures, beyond holistic neighbor aggregation.

On larger and structurally more diverse datasets such as NCI1 and COLLAB, MSH-GNN continues to outperform both attention-based and pooling-based methods, including GMT and SEP. This suggests that the proposed node-adaptive harmonic filtering generalizes beyond small, highly regular graphs and remains effective in heterogeneous settings with varied local connectivity patterns. Compared with topology-driven pooling approaches such as Wit-TopoPool, which rely on explicit subgraph constructions, MSH-GNN achieves comparable or better accuracy while retaining a fully differentiable, end-to-end message passing architecture with lower structural overhead.

Overall, these results demonstrate that incorporating feature-wise,
frequency-aware message construction provides a robust inductive bias
that benefits graph classification across both domain-specific and
large-scale real-world benchmarks.

\subsection{Node Classification}

To further evaluate the generality and scalability of MSH-GNN beyond small, motif-rich molecular graphs, we conduct node classification experiments on several large-scale benchmark datasets, including Coauthor Physics, Amazon Photo, Amazon Computers, and Minesweeper. These datasets exhibit diverse structural characteristics, ranging from community-driven citation networks to synthetic graphs with explicitly planted local motifs, providing a complementary testbed to assess feature-wise and frequency-aware message passing.

We follow the experimental protocol of~\cite{sun2024towards}, using public data splits and reporting the average classification accuracy over multiple random seeds. The results are summarized in Table~\ref{tab:small_scale_results}.

\begin{table*}[t]
\centering
\begin{tabular}{lcccc}
\toprule
 & \textbf{CoauthorPhy} & \textbf{AmzPhoto} & \textbf{AmzComputers} & \textbf{Minesweeper} \\
\midrule
No. nodes & 34,493 & 7,487 & 13,381 & 10,000 \\
No. edges & 247,962 & 119,043 & 245,778 & 39,402 \\
\midrule
GCN~\cite{kipf2016semi}  & 96.18$\pm$0.07 & 92.70$\pm$0.20 & 89.65$\pm$0.52 & 89.75$\pm$0.52 \\
GAT~\cite{velivckovic2017graph}  & 96.17$\pm$0.08 & 93.87$\pm$0.11 & 90.78$\pm$0.17 & 92.01$\pm$0.68 \\
GPRGNN~\cite{chien2020adaptive}  & 96.85$\pm$0.08 & 94.49$\pm$0.14 & 89.32$\pm$0.29 & 86.24$\pm$0.61 \\
APPNP~\cite{gasteiger2018predict}  & 96.54$\pm$0.07 & 94.32$\pm$0.14 & 90.18$\pm$0.17 & -- \\
\midrule
GT~\cite{shi2020masked}  & 97.05$\pm$0.05 & 94.74$\pm$0.13 & 91.18$\pm$0.17 & 91.85$\pm$0.76 \\
Graphormer~\cite{ying2021transformers} & OOM & 92.74$\pm$0.14 & OOM & OOM \\
SAN~\cite{kreuzer2021rethinking} & OOM & 94.86$\pm$0.10 & 89.83$\pm$0.16 & -- \\
GraphGPS~\cite{rampavsek2022recipe} & OOM & 95.06$\pm$0.13 & OOM & 92.29$\pm$0.61 \\
NAG\textsubscript{phormer}~\cite{chen2022nagphormer} & 97.34$\pm$0.03 & 95.49$\pm$0.11 & 91.22$\pm$0.14 & -- \\
Exphormer~\cite{shirzad2023exphormer} & 97.16$\pm$0.13 & 95.27$\pm$0.42 & 91.59$\pm$0.31 & 90.42$\pm$0.10 \\
N$^2$~\cite{sun2024towards} & 97.56$\pm$0.28 & 95.75$\pm$0.34 & 92.51$\pm$0.13 & 93.97$\pm$0.27 \\
\midrule
\textbf{MSH-GNN} & \textbf{98.37$\pm$0.08} & \textbf{97.66$\pm$0.50} & \textbf{92.84$\pm$1.61} & \textbf{94.78$\pm$0.64} \\
\bottomrule
\end{tabular}
\\
\caption{Node classification results on benchmark graphs.
Results are reported as accuracy (\%) $\pm$ standard deviation.}
\label{tab:small_scale_results}
\end{table*}

As shown in Table~\ref{tab:small_scale_results}, MSH-GNN achieves state-of-the-art performance across all evaluated datasets. Notably, on {Minesweeper}, a synthetic graph with explicitly planted local motifs, MSH-GNN outperforms the strong N$^2$ baseline by {+0.81\%}, highlighting its ability to capture localized, frequency-sensitive patterns that are challenging for conventional message passing schemes.

On real-world graphs such as {Amazon Photo} and {Coauthor Physics}, where community structure and role-based connectivity dominate, MSH-GNN consistently improves over both spectral GNNs and transformer-based models. Importantly, unlike global-attention architectures such as Graphormer and GraphGPS, which frequently encounter out-of-memory (OOM) issues on large graphs, MSH-GNN maintains a lightweight, message-passing-based design, achieving strong accuracy while remaining computationally efficient and scalable.

\begin{figure*}[htbp]
    \centering
    \includegraphics[width=1.0\linewidth]{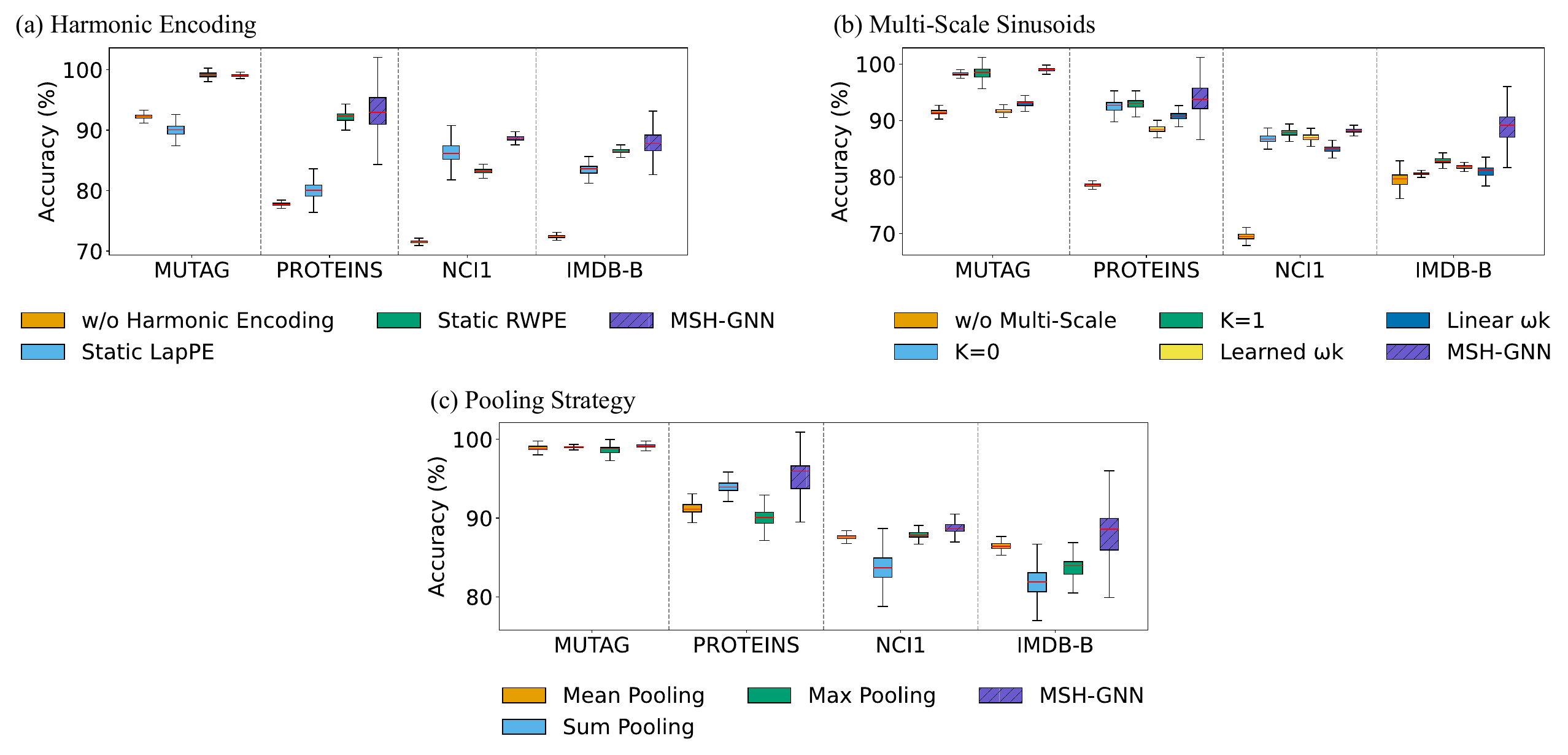}
    \caption{
Ablation results on representative TU datasets.
Panels (a)--(c) analyze the impact of harmonic encoding,
multi-scale sinusoidal modulation, and pooling strategies,
respectively, highlighting the effectiveness of each design choice in MSH-GNN.
}
    \label{fig_3:ablation}
\end{figure*}

\subsection{Ablation Study}
We conduct an ablation study to quantify the contribution of each key component in MSH-GNN, including harmonic encoding, multi-scale sinusoidal modulation, and the pooling strategy. Results on representative TU datasets are illustrated in Fig.~\ref{fig_3:ablation}, which provides distributional comparisons across different ablation settings.

\paragraph{Effect of Harmonic Encoding}
Removing the harmonic encoding module leads to a substantial performance degradation,
particularly on PROTEINS and NCI1, highlighting the importance of frequency-aware message construction.
We further compare the proposed harmonic encoding with static positional encodings,
including Laplacian positional encoding (LapPE) and random-walk-based positional encoding (RWPE).

While these spectral priors improve performance over the baseline without harmonic encoding, they remain fixed across nodes and layers, limiting their ability to adapt to local structural variations. In contrast, the proposed harmonic encoding dynamically generates node-conditioned projections prior to sinusoidal modulation, enabling frequency responses to adapt to the target node and its neighborhood context.

This adaptive design consistently outperforms static encodings across datasets, indicating that frequency information alone is insufficient without target-dependent feature alignment.

\paragraph{Effect of Multi-Scale Sinusoidal Modulation}
We further evaluate different frequency configurations, including no modulation, single-scale variants ($K=0$ or $K=1$), learned frequencies, linear schedules, and the default exponential schedule ($\omega_k = 2^k$). Across all datasets, incorporating multi-scale sinusoidal modulation consistently improves performance, with the exponential schedule achieving the best overall results.

Compared to learned or linearly spaced frequencies, the exponential schedule provides clearer separation across frequency bands, which appears beneficial for capturing both smooth and oscillatory structural patterns. These results suggest that appropriate frequency spacing plays a critical role in effective harmonic modulation, supporting robust representation learning across diverse graph structures.

\paragraph{Effect of Pooling Strategy}
We compare mean, sum, and max pooling with the proposed frequency-aware attention pooling strategy. Attention-based pooling consistently yields superior performance, particularly on NCI1 and IMDB-B, where identifying structurally informative nodes is crucial. This observation indicates that selectively aggregating harmonically informative responses at the readout stage is important for preserving discriminative structure-level information.

Overall, unlike static Laplacian-based encodings such as LapPE and RWPE, the proposed harmonic module performs node-conditioned projection prior to modulation, aligning frequency information with the local message context. This target-adaptive alignment mitigates feature misalignment caused by fixed encodings, thereby enabling flexible and structure-aware representation learning across diverse graph domains.

\begin{figure*}[htbp]
    \centering    \includegraphics[width=0.9\linewidth]{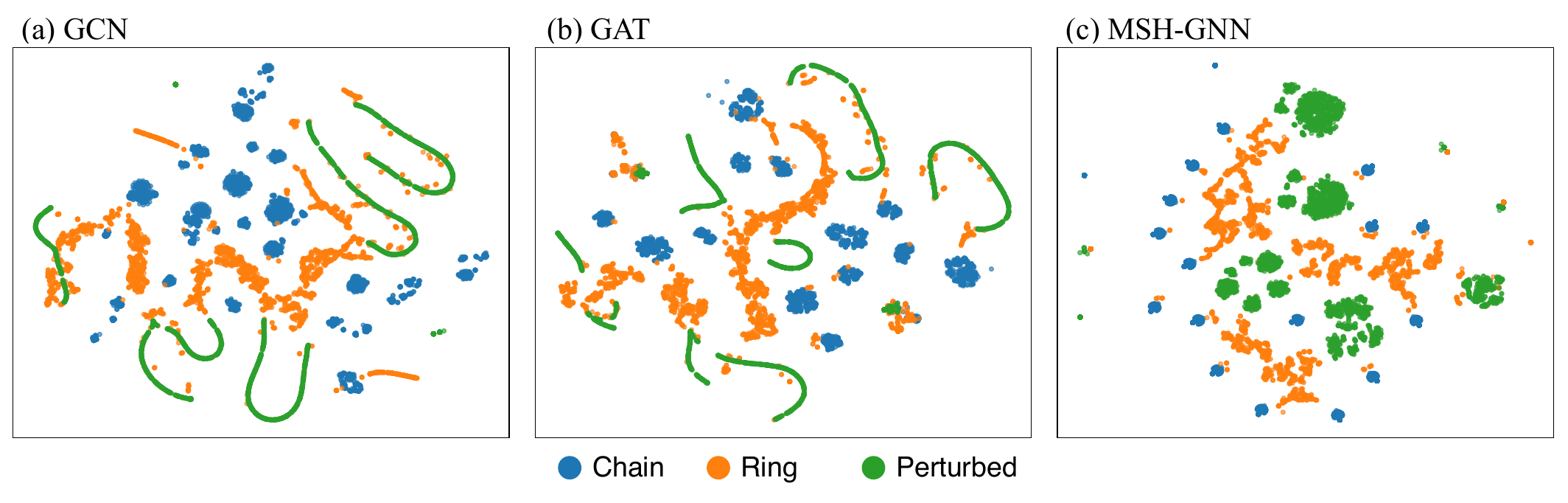}
    \caption{Graph-level embedding visualization via t-SNE. Each point represents one graph, colored by structural type (ring / chain / perturbed). MSH-GNN produces more separable clusters, indicating better structure–frequency representation.}
    \label{fig_4:tsne}
\end{figure*}

\section{Joint Frequency–Structure Classification}
To explicitly evaluate the ability of MSH-GNN to jointly model structural patterns and spectral characteristics, we design a synthetic graph classification task with 30 distinct classes. Each class corresponds to a unique combination of graph topology and graph spectral mode.

\paragraph{Graph Construction}
We generate synthetic graphs from three structural families:
\begin{itemize}
    \item {Ring graphs}: $n$-node cycle graphs with regular connectivity.
    \item {Chain graphs}: path graphs of the same size, exhibiting linear topology.
    \item {Perturbed rings}: ring graphs with $20\%$ of edges randomly rewired to introduce controlled structural noise.
\end{itemize}

\paragraph{Node Features}
For each generated graph, we compute the unnormalized graph Laplacian $L = D - A$ and extract its eigenvectors $\mathbf{u}_0, \mathbf{u}_1, \dots, \mathbf{u}_{n-1}$, ordered by increasing eigenvalues. We then assign node features using the $k$-th eigenvector $\mathbf{u}_k$, where $k \in \{0, 1, \dots, 9\}$.

Each eigenvector serves as a synthetic graph signal corresponding to a specific graph Fourier mode: smaller $k$ indices represent low-frequency (smooth) variations over the graph, while larger $k$ indices correspond to higher-frequency (oscillatory) patterns.

\paragraph{Classification Task}
Each graph is labeled by its \emph{structure--frequency pair} $(S, k)$, where $S \in \{\text{ring}, \text{chain}, \text{perturbed}\}$ denotes the structural family and $k$ indexes the selected spectral mode. This results in a 30-class classification problem (3 structures $\times$ 10 spectral modes), requiring models to simultaneously distinguish both topological and spectral variations.

\paragraph{Results}
We evaluate classification accuracy together with model complexity, inference time, and GPU memory usage. As reported in Table~\ref{tab:complexity_1}, MSH-GNN achieves substantially higher accuracy than GCN and GAT on this joint structure--frequency task, while incurring only modest additional computational overhead.

\begin{table}[htbp]
\centering
\small
\begin{tabular}{lcccc}
\toprule
\textbf{Model} & \#Params & Inference (s) & GPU (MB) & Accuracy \\
\midrule
GCN      & 6.2K & 1.40 & 17.90 & 0.5837 \\
GAT      & 6.5K & 2.50 & 17.95 & 0.5593 \\
MSH-GNN  & 7.3K & 2.74 & 26.64 & \textbf{0.7809} \\
\bottomrule
\end{tabular}
\caption{Model complexity and efficiency comparison on the synthetic classification task.}
\label{tab:complexity_1}
\end{table}

\paragraph*{Visualization}
To further examine whether MSH-GNN captures discriminative representations across both structural types and spectral modes, we perform a post-hoc visualization analysis.

As shown in Figure~\ref{fig_4:tsne}, embeddings produced by MSH-GNN form more clearly separated clusters corresponding to different structure--frequency combinations. In contrast, GCN and GAT exhibit more entangled distributions, indicating limited capacity to jointly distinguish structural and spectral variations. These observations support that the proposed harmonic encoding enhances representation learning for graphs with diverse spectral characteristics.

\section{Conclusion and Future Work}
We proposed MSH-GNN, a graph neural network architecture that integrates node-conditioned harmonic projections with multi-scale sinusoidal modulation for frequency-aware, feature-wise message passing. By incorporating spectral inductive biases directly into message construction, MSH-GNN enables target nodes to selectively filter task-relevant feature components from their neighbors, going beyond conventional GNNs that operate on holistic node embeddings.

From a theoretical perspective, we show that the proposed harmonic encoding admits a kernel-based interpretation and provides structure-adaptive, frequency-selective message functions. Under mild injectivity conditions, MSH-GNN achieves expressive power comparable to the 1-Weisfeiler--Lehman test while operating as a frequency-aware graph operator.

Empirically, MSH-GNN demonstrates consistently strong performance across molecular, social, and synthetic graph benchmarks. In particular, experiments on joint structure--frequency classification tasks highlight its ability to capture both topological patterns and spectral variations, outperforming standard GNN baselines such as GCN and GAT.

Future work includes extending the framework to equivariant or invariant settings, incorporating adaptive frequency scales, and exploring hierarchical spectral modeling at the subgraph level. Overall, MSH-GNN provides a principled and scalable approach to frequency-aware graph representation learning.

\appendices

\section{A Constructive Example Illustrating the Limitation of GAT}
\label{appendix:gat-gap}

To highlight the expressiveness advantage of MSH-GNN over attention-based GNNs, we construct a pair of graphs $G_1$ and $G_2$ that are indistinguishable by GAT but clearly separable by MSH-GNN.

Consider two star graphs shown in Figure~\ref{fig:enter-label}, each with a central node $v$ and three neighbors $u_1$, $u_2$, $u_3$. Both graphs satisfy:
\begin{itemize}
    \item The same central node feature: $\mathbf{h}_v = [0, 0]$;
    \item The same mean of neighbor features: $\frac{1}{3} (\mathbf{h}_{u_1} + \mathbf{h}_{u_2} + \mathbf{h}_{u_3}) = [0.67, 0.67]$;
    \item But distinct directional arrangement of neighbor features.
\end{itemize}

\begin{figure}[htbp]
    \centering
    \includegraphics[width=0.6\linewidth]{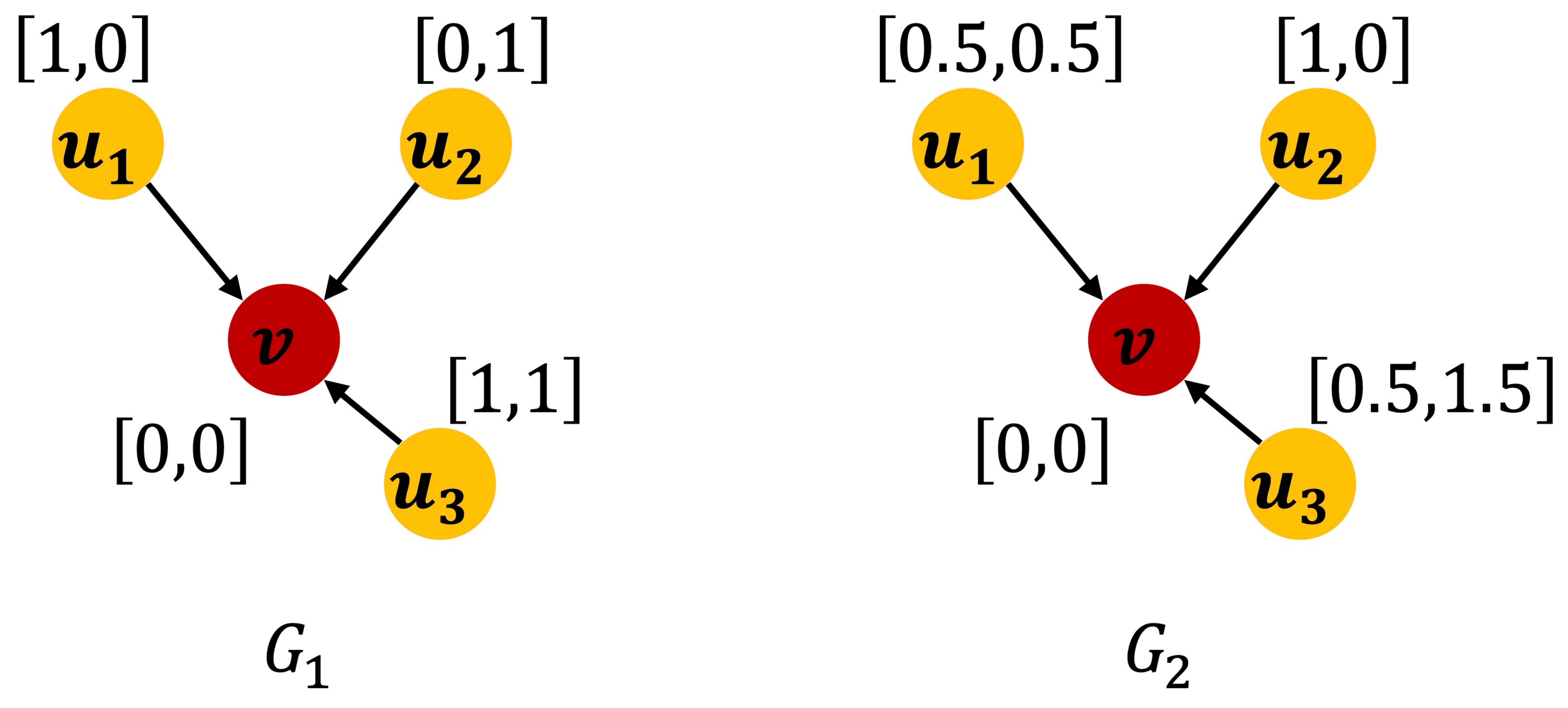}
    \caption{Two graphs $G_1$ and $G_2$ with identical neighbor feature mean but different feature-space composition. GAT cannot distinguish them due to scalar attention over holistic embeddings. MSH-GNN captures directional asymmetry via node-specific harmonic projections.}
    \label{fig:enter-label}
\end{figure}

\paragraph{GAT Tends to Produce Identical Aggregations for $G_1$ and $G_2$}
GAT computes scalar attention scores $\alpha_i$ based on the compatibility between $h_v$ and $h_{u_i}$, and performs the following aggregation:
\[
\mathbf{h}_v' = \sum_{i=1}^3 \alpha_i \mathbf{h}_{u_i}, \quad \sum_i \alpha_i = 1.
\]
Since the average neighbor feature is the same and $h_v = [0, 0]$, GAT produces similar outputs for $G_1$ and $G_2$. It is not explicitly designed to differentiate directional arrangements or intra-feature asymmetry.

\paragraph*{MSH-GNN Captures Directional Differences}
In contrast, MSH-GNN modulates messages using target-dependent projections and multi-scale harmonics:
\begin{align*}
\mathbf{p}_{vu} &= \mathbf{F}_v \mathbf{h}_{u} + \boldsymbol{\phi}_v, \\
\psi_k(\mathbf{p}_{vu}) &= \sin(2^k \mathbf{p}_{vu}) \,\|\, \cos(2^k \mathbf{p}_{vu}), \\
\Psi_{vu}(\mathbf{h}_v, \mathbf{h}_u) &= \mathbf{W}_o \left[ \psi_1(\mathbf{p}_{vu}) \,\|\, \psi_2(\mathbf{p}_{vu}) \,\|\, \psi_3(\mathbf{p}_{vu}) \right].
\end{align*}

Even if neighbor averages are equal, the node-specific projections $\mathbf{F}_v \mathbf{h}_u$ preserve directional information, and harmonic modulation accentuates these differences at multiple frequency scales. As a result, $\Psi_{vu}$ differs between $G_1$ and $G_2$ despite equal mean inputs, enabling MSH-GNN to distinguish the graphs. As a result, MSH-GNN can produce distinct messages for $G_1$ and $G_2$ via:
\begin{itemize}
    \item Direction-sensitive basis projection;
    \item Phase-aware modulation;
    \item Nonlinear frequency filtering that breaks permutation symmetry in feature space.
\end{itemize}

This example illustrates the failure of GAT under structural aliasing and the power of MSH-GNN to overcome it via feature-wise, frequency-aware message construction.

\section{Theoretical Justification}
\label{appendix:kernel-approximation}

\subsection{Harmonic Modulation as a Kernel Approximation}

We provide an auxiliary interpretation of the harmonic modulation in MSH-GNN
as a structured Fourier-feature map, which induces a kernel-like similarity
over \emph{node-conditioned projected features}. This perspective helps explain
why the proposed message function can capture rich nonlinear interactions in a
localized and target-dependent manner.

\subsubsection{Bochner's Theorem and Fourier Feature Maps}

Bochner's theorem states that a continuous, shift-invariant kernel
$k(x,x') = k(x-x')$ on $\mathbb{R}^d$ is positive definite if and only if it is
the Fourier transform of a non-negative finite measure $\mu$:
\begin{equation}
k(x-x') = \int_{\mathbb{R}^d} e^{i \eta^\top (x-x')} \, d\mu(\eta).
\end{equation}
When $\mu$ admits a density $p(\eta)$, a standard construction yields Fourier
feature maps (random or deterministic) whose inner products approximate
$k(x-x')$ \cite{rahimi2007random}.

\subsubsection{MSH-GNN as Deterministic Multi-Frequency Fourier Features}

Recall from Section~\ref{sec:methods} that, at layer $l$, each target node $v$
generates a projection matrix $\mathbf{F}_v \in \mathbb{R}^{F\times d}$ and a
phase vector $\boldsymbol{\phi}_v \in \mathbb{R}^{F}$. For each directed edge
$(u\!\rightarrow\! v)$, MSH-GNN forms a receiver-conditioned projection
\begin{equation}
\mathbf{p}_{vu} = \mathbf{F}_v \mathbf{h}_u + \boldsymbol{\phi}_v \in \mathbb{R}^{F},
\end{equation}
and applies a multi-frequency sinusoidal map
\begin{equation}
\psi(\mathbf{p}_{vu})
= \big\|_{k=0}^{K}
\left[\sin(\omega_k \mathbf{p}_{vu}) \,\|\, \cos(\omega_k \mathbf{p}_{vu})\right],
\qquad \omega_k > 0.
\end{equation}
In our implementation, $\{\omega_k\}$ follows an exponential schedule (e.g.,
$\omega_k = 2^{k}$), producing a compact multi-band feature expansion.

This construction can be viewed as a \emph{deterministic Fourier-feature map}
defined on node-conditioned projected inputs. Consequently, for a fixed target
node $v$, the inner product
\begin{equation}
\psi(\mathbf{p}_{vu})^\top \psi(\mathbf{p}_{vu'})
\end{equation}
induces a \emph{target-dependent} similarity function over neighbors
$\mathbf{h}_u$ and $\mathbf{h}_{u'}$. This motivates defining an implicit kernel
\begin{equation}
k_v(\mathbf{h}_u,\mathbf{h}_{u'})
:= \psi(\mathbf{F}_v\mathbf{h}_u+\boldsymbol{\phi}_v)^\top
   \psi(\mathbf{F}_v\mathbf{h}_{u'}+\boldsymbol{\phi}_v),
\end{equation}
which resembles a shift-invariant kernel \emph{in the projected feature space}
(up to the choice of frequencies and finite truncation). Unlike globally sampled
random Fourier features, the map here is \emph{receiver-conditioned} through
$\mathbf{F}_v$ and $\boldsymbol{\phi}_v$, enabling localized and adaptive
nonlinear filtering.

\begin{remark}
The kernel interpretation above is not used as an algorithmic component;
rather, it provides an explanatory lens for understanding the harmonic module
as an inner-product feature map in a high-dimensional sinusoidal space.
\end{remark}

\subsection{Spectral Modulation and Frequency Sensitivity}

We next provide an intuitive connection between the harmonic module and
frequency sensitivity of graph signals. Classical message passing layers (e.g.,
GCN/GIN) apply shared transformations that do not explicitly separate responses
across different spectral modes, while attention mechanisms primarily reweight
neighbors. In contrast, MSH-GNN applies \emph{multi-frequency} nonlinear maps on
receiver-conditioned projections, which encourages distinct responses to input
signals of different spectral variability.

Concretely, if an input signal varies slowly across the graph (low-frequency),
then $\mathbf{p}_{vu}$ tends to vary smoothly and low $\omega_k$ components
provide stable encodings. If an input signal exhibits rapid local variations
(high-frequency), higher $\omega_k$ components produce more oscillatory
responses. This multi-band behavior acts as a lightweight filter bank in the
learned feature space and complements the node-conditioned projection mechanism.

\begin{remark}
Here ``frequency'' is used in the sense of variability of node signals over the
graph. MSH-GNN does not parameterize filters directly on the Laplacian
eigenbasis; instead, it uses multi-frequency sinusoidal mappings on projected
features, which empirically yields frequency-sensitive representations.
\end{remark}

\subsection{Injective Message Functions Imply 1-WL-Level Expressiveness}
\label{appendix:proof}

We provide a stricter (sufficient) condition under which MSH-GNN matches the
discriminative power of the 1-Weisfeiler--Lehman (1-WL) test, following the
injectivity framework of GIN \cite{xu2018powerful}.

\begin{theorem}[Sufficient Conditions for 1-WL-Level Expressiveness]
Let $\mathcal{F}: G \rightarrow \mathbb{R}^n$ be an instance of MSH-GNN. Assume:
\begin{itemize}
    \item Initial node features are injective (or augmented with unique
    identifiers / sufficiently informative structural encodings);
    \item $\mathbf{F}_v$ and $\boldsymbol{\phi}_v$ are generated from $\mathbf{h}_v$
    by sufficiently expressive MLPs, and the resulting receiver-conditioned
    message map is injective in $\mathbf{h}_u$ conditioned on $\mathbf{h}_v$;
    \item The harmonic map uses distinct non-zero frequencies $\{\omega_k\}$ and
    the post-harmonic projection $\mathbf{W}_o$ does not collapse different
    sinusoidal codes;
    \item The neighborhood aggregation and node update are permutation-invariant
    and sufficiently expressive (e.g., sum/mean aggregation followed by an MLP).
\end{itemize}
Then, $\mathcal{F}$ is at least as powerful as the 1-WL test in distinguishing
non-isomorphic graphs.
\end{theorem}

\begin{proof}
Consider the edge message in MSH-GNN:
\begin{equation}
\begin{split}
\Psi_{vu}(\mathbf{h}_v,\mathbf{h}_u)
&= \mathbf{W}_o \Big[ \big\|_{k=0}^{K}
    \big(\sin(\omega_k\mathbf{p}_{vu}) \,\|\, \cos(\omega_k\mathbf{p}_{vu})\big) \Big], \\
\mathbf{p}_{vu}
&= \mathbf{F}_v\mathbf{h}_u + \boldsymbol{\phi}_v.
\end{split}
\end{equation}

By assumption, for a fixed $v$, the composed map
$\mathbf{h}_u \mapsto \Psi_{vu}(\mathbf{h}_v,\mathbf{h}_u)$ is injective.
Therefore, different multisets of neighbor features induce different multisets
of messages $\{\Psi_{vu}\}_{u\in \mathcal{N}(v)}$.

Since the aggregation is permutation-invariant and the subsequent node update
is sufficiently expressive (as in the GIN analysis \cite{xu2018powerful}),
the updated node embedding $\mathbf{h}_v^{(l+1)}$ can be made to be an injective
function of the multiset of neighbor features, matching the refinement step of
1-WL. Repeating this argument across layers yields 1-WL-level distinguishability.
\end{proof}

\begin{remark}
The conditions above are sufficient and intentionally conservative. In practice,
exact injectivity is not required for strong empirical performance; the
receiver-conditioned projection and multi-frequency harmonic expansion provide a
useful inductive bias for encoding neighborhood multisets in a frequency-aware,
feature-wise manner.
\end{remark}

\noindent {Summary}
The harmonic module admits a kernel-style interpretation as an inner-product
feature map on node-conditioned projections, and the resulting message function
can satisfy the standard injectivity conditions that imply 1-WL-level
expressiveness.

\section{Supplementary Experimental Settings}
\label{appendix:experiments}

\subsection{Node Classification Benchmarks}
\label{appendix:node}

To complement the graph-level evaluations in the main text, we additionally assess \textbf{MSH-GNN} on standard node classification benchmarks, including \textit{Coauthor Physics}, \textit{Amazon Computers}, \textit{Amazon Photo}, and \textit{Minesweeper}. These datasets span co-authorship networks, e-commerce graphs, and synthetic graphs with varying structural regularities.

All experiments follow the training protocol of N$^2$~\cite{sun2024towards}, using publicly available data splits and reporting the average classification accuracy over multiple random seeds.

\paragraph{Baselines}
We compare MSH-GNN against representative message-passing GNNs and graph transformer models:

\begin{itemize}
    \item \textbf{Message Passing GNNs}: GCN~\cite{kipf2016semi}, GAT~\cite{velivckovic2017graph}, APPNP~\cite{gasteiger2018predict}, GPR-GNN~\cite{chien2020adaptive};
    \item \textbf{Graph Transformers}: GT~\cite{shi2020masked}, Graphormer~\cite{ying2021transformers}, SAN~\cite{kreuzer2021rethinking}, GraphGPS~\cite{rampavsek2022recipe};
    \item \textbf{Spectral-aware Transformers}: NAG\textsubscript{phormer}~\cite{chen2022nagphormer}, Exphormer~\cite{shirzad2023exphormer}, and N$^2$~\cite{sun2024towards}.
\end{itemize}

\subsection{Hyperparameter Configuration}

We tune hyperparameters via grid search for each dataset. The final configurations are summarized in Table~\ref{tab:hyperparams}. Unless otherwise specified, all models are trained using the Adam optimizer with a learning rate of $0.001$ and dropout rate $0.1$. The hidden and output dimensions are set to $64$ and $16$, respectively.

\begin{table}[htbp]
\centering
\small
\begin{tabular}{lcccc}
\toprule
\textbf{Dataset} & \#Layers & Hidden$_1$ & Hidden$_2$ & Dropout \\
\midrule
TU datasets      & 3  & 64 & 16 & 0.1 \\
Coauthor Physics & 3  & 64 & 16 & 0.1 \\
Amazon Computers & 3  & 64 & 16 & 0.1 \\
Amazon Photo     & 3  & 64 & 16 & 0.1 \\
Minesweeper      & 15 & 64 & 16 & 0.1 \\
\bottomrule
\end{tabular}
\caption{Hyperparameter configurations used in the experiments.}
\label{tab:hyperparams}
\end{table}

\bibliographystyle{IEEEtran}
\bibliography{references}

\end{document}